\documentclass[letterpaper, 10 pt, conference]{ieeeconf}

\IEEEoverridecommandlockouts                              

\usepackage[T1]{fontenc} 
\usepackage{amsmath}

\usepackage{amsfonts}
\usepackage{amssymb}
\usepackage{dsfont}
\usepackage{graphicx,color}
\usepackage[first=-100, last=100]{lcg}
\usepackage{cleveref,enumerate}
\usepackage{tikz}
\usetikzlibrary{shapes,arrows}
\usepackage{subcaption}
\usepackage{cite}
\usepackage{ntheorem}

\newcommand{\R}{\mathbb{R}}
\newcommand{\G}{\mathcal{G}}
\newcommand{\E}{\mathcal{E}}
\newcommand{\V}{\mathcal{V}}
\newcommand{\N}{\mathcal{N}}
\renewcommand{\L}{\mathcal{L}}

\newcommand{\diag}{\text{diag}}
\renewcommand{\span}{\text{span}}

\newcommand{\ones}{\mathbf{1}}

\newtheorem{lemma}{\textbf{Lemma}}
\newtheorem{theorem}{\textbf{Theorem}}
\newtheorem{problem}{\textbf{Problem}}

\newtheorem{coro}{\textbf{Corollary}}
\newcounter{example}
\newenvironment{example}[1][]{\refstepcounter{example}\par\medskip
	\noindent \textbf{Example~\theexample. #1} \rmfamily}{\medskip}
\newcounter{remark}
\newenvironment{remark}[1][]{\refstepcounter{remark}\par\medskip
	\noindent \textit{Remark~\theremark. #1} \rmfamily}{\medskip}
\newtheorem{assumption}{\textbf{Assumption}}

\newcounter{definition}
\newenvironment{definition}[1][]{\refstepcounter{definition}\par\medskip
	\noindent \textbf{Definition~\thedefinition. #1} \rmfamily}{\medskip}
\author{Mehran Zareh, Lorenzo Sabattini, and Cristian Secchi
	\thanks{ 
		Authors are with the Department of Sciences and Methods for Engineering (DISMI), University of Modena and Reggio Emilia, Italy {\tt\small{\{mehran.zareh, lorenzo.sabattini, cristian.secchi\}@unimore.it}}}
}
\date{}
\title{\LARGE \bf Enforcing Biconnectivity in Multi-robot Systems}

\begin{document}
		\maketitle
\thispagestyle{empty}
\pagestyle{empty}

\begin{abstract}
Connectivity maintenance is an essential task in multi-robot systems and it has  received a considerable attention during the last years. A connected system can be  broken into two or more subsets simply if a single robot fails. A more robust communication can be achieved if the  network connectivity is guaranteed in the case of one-robot failures. The resulting network is called biconnected. In \cite{Zareh2016biconnectivitycheck}, we presented a criterion for biconnectivity check, which basically determines a lower bound on the third-smallest eigenvalue of the Laplacian matrix.  In this paper, we introduce a decentralized gradient-based protocol to increase the value of the third-smallest eigenvalue of the Laplacian matrix, when the biconnectivity check fails.  We also introduce a decentralized algorithm to estimate the eigenvectors of the Laplacian matrix, which are used for defining the gradient. Simulations show the effectiveness of the theoretical findings.
\end{abstract}
	\section{Introduction}
	In the last decade, decentralized control systems have been increasingly investigated \cite{olfati2005consensus, Zareh_consensus,ZarehPhd15}. Advances in small size computation, communication, sensing, and actuation have caused a growing interest in decentralized control and decision making.
	Decentralized control of multi-robot systems can be exploited for addressing many real world applications (e.g., surveillance, exploration of unknown environments, space-based interferometers, and automatic highways). In these systems the robots coordinate their motion, in order to achieve the global objective.  Because of some unknown obstacles, the robots might get trapped and hence disconnected from the team. Therefore, the robots must recognize these phenomena and utilize proper strategies to preserve the network connectivity. This is a substantial task that must be seen as an objective of the control action.  In the existing literature on multi-robot control systems, the connectivity of the network graph, i.e. the interaction pattern among the robots, is assumed. There are two main approaches to preserve the connectivity: local and global maintenance.
	In local connectivity maintenance the aim is to develop a controller that keeps  all initially existing communication links. Some examples of decentralized control for local connectivity maintenance can be found in \cite{notarstefano2006maintaining, ajorlou2010class}. In comparison to the local ones, the global connectivity maintenance algorithms are based on global quantities of the network, and do not restrict link failures or creation. In the last few years, several works on this topic (see e.g. \cite{zavlanos2011graph,sabattini2013decentralized,sabattini2013distributed,giordano2013passivity}) have appeared.

	To obtain a robust communication in a multi-robot system, the connectivity has to be guaranteed when a single robot crashes or is suddenly called by a human user to perform some unpredicted task. In other words, the resulting graph must remain connected if one of the nodes and all its incident edges are removed. Possessing this property, the graph is called biconnected \cite{golumbic2004algorithmic}.  In addition to robustness, biconnectivity provides a better bandwidth for communication by providing multiple paths to the destination. The connectivity robustness of robot networks under failures is often neglected in the literature. Some related works in graph theory describe algorithms to find biconnected components in a graph based on optimization theories. These algorithms mainly utilize depth-first search or backtracking \cite{tarjan1972depth, tarjan1984finding} in a centralized way. In \cite{ahmadi2006distributed,ahmadi2006keeping}, the problem of biconnectivity check for a network is presented. They propose an approach to detect the biconnected component. Since the algorithm requires a global probe, it cannot be seen as a decentralized one. Very recently, \cite{ghedini2015improving} investigated the robustness problem in multi-robot systems so that, despite robot failures, most of the robots remain  connected and are able to continue the mission. Based on a maximum 2-hop communication, each robot is able to detect dangerous topological configurations in the sense of the connectivity and can mitigate in order to reach a new position to get a better connectivity level. The paper, based on local information, introduces a parameter, called vulnerability, that allows each robot to detect the level of its effect on the topological configuration.
	
	In order to have a biconnected network graph, one needs to recognize and relocate the robots, whose failure potentially can cause disconnection, so that more connections are created. In \cite{Zareh2016biconnectivitycheck}, based on a decentralized algorithm, we proved that the biconnectivity conditions are related to the third-smallest eigenvalue of the Laplacian matrix.

	 In this paper, we propose an algorithm to enable each node of the network graph to detect if it is a crucial one for the network connectivity.  These nodes are termed as articulation points. If there is no articulation point, then the resulting graph is biconnected. Moreover, we provide an algorithm for enforcing biconnectivity. First, each robot perturbs its communication links' weights, estimates the eigenvalues of the perturbed Laplacian matrix, and checks the biconnectivity condition introduced in \cite{Zareh2016biconnectivitycheck}. Then, if the check fails, the robots starts moving to new positions to create new links. The main idea is to form a gradient-based controller to increase the third-smallest eigenvalue of the Laplacian matrix. To this end, we need to have decentralized estimates of the third-smallest eigenvalue and an associated eigenvector. For eigenvalue estimation we use the algorithm introduced by \cite{franceschelli2013decentralized}. We develop a decentralized protocol that allows each robot to estimate the eigenvectors of the Laplacian matrix.

	The outline of the paper is as follows. In Section~\ref{section:notations}, we introduce notations and some basic notions on graph theory, which will be used in this work. The problem statement is introduced in Section~\ref{section:problem_statement}. Section~\ref{section:main_results} provides the main contribution of this paper. We provide some theorems on decentralized eigenvector and eigenvalue estimation, and a gradient-based controller to achieve biconnectivity. In Section~\ref{section:simulations}, the simulation results are given to verify the theoretical findings. Finally, in Section~\ref{section:conclusions}, we conclude the paper and describe the open problems.

	
	\section{Preliminaries}\label{section:notations}
	
	In this section, we recall some basic notions and definitions on graph theory and introduce the notation used in the paper.
	
	The topology of bidirectional communication channels among the robots is represented by an undirected graph $\G(\V, \E)$ where $\V=\{1,\ldots,n\}$ is the set of nodes (robots) and $\E\subset\V\times \V$ is the set of edges. An edge $(i, j) \in  \E$ exists if there is a
	communication channel between robots $i$ and $j$. Self loops
	$(i, i)$ are not considered. The set of robot $i$'s neighbors
	is denoted by $\N_i  = \{j \ : \ (j, i) \in  \E; j = 1, \ldots, n\}$. 
	The network graph $\G$ is encoded by the so-called {\em adjacency matrix}, an $n \times n$ matrix $A$ whose $(i,j)$-th entry $a_{ij}$ is greater than $0$ if $(i,j)\in \E$, $0$ otherwise. Obviously in an undirected graph matrix $A$ is symmetric.
	The degree matrix is defined as $D=\diag(d_1,d_2, \ldots, d_n)$ where	$d_i = \sum_{j=1}^{n}a_{ij}$ is the degree of node $i$. The Laplacian matrix of a graph is defined as $\L=D-A$. The $i$-th column of $\L$ is denoted by $l_i$. 
	The Laplacian matrix of a graph has several structural properties. Due to the Gershgorin Circle Theorem \cite{meyer2000matrix} applied to the rows or the columns of the Laplacian, it is possible to show that it has non-negative real eigenvalues for any undirected graph $\G$. By construction matrix $\L$ has at least one null eigenvalue because either the row sum or the column sum is zero. Furthermore, let $\ones$ and $\mathbf{0}$ be respectively the vectors of ones and zeros with proper dimensions, then $\L\ones=\mathbf{0}$ and $\ones^T\L=\mathbf{0}^T$. Denote by $\lambda_i(\cdot)$ the $i$-th smallest eigenvalue of a matrix, and $v_i(\cdot)$ an associated right eigenvector.  Due to the symmetry, the eigenvalues of the Laplacian matrix are all real, and can be ordered as
	$$0=\lambda_1(\L)\le\lambda_2(\L)\le \ldots\le \lambda_n(\L).$$
	In $\G$ a node $i$ is reachable from a node $j$ if there exists an undirected path from $j$ to $i$ or vice versa. If $\G$ is connected then $\L$ is a symmetric positive semidefinite irreducible matrix.  Moreover, the algebraic multiplicity of the null eigenvalue of $\L$ is one.
	For a graph $\G$, the second smallest eigenvalue of the Laplacian matrix is called \emph{algebraic connectivity}. This eigenvalue gives a measure of connectedness of the graph. 
	Algebraic connectivity is a non-decreasing function of graphs with the same set of vertices. This means that if $\G_1(\V, \E_1)$ and $\G_2(\V, \E_2)$ are two graphs constructed on the set $\V$  such that $\E_1 \subseteq \E_2$, then $\lambda_2(\G_1)\leq\lambda_2(\G_2)$. 
	In other words,  the more connected the graph becomes the larger the algebraic connectivity will be.

	We denote $\tilde{a}_i=[a_{ij}]^T\in \R^{n}, \ \ j=1,\ldots,n,j\ne i$. We also define the perturbed adjacency matrix $A^i(\epsilon)$ obtained from $A$ by multiplying all $a_{ij}$ and $a_{ji}$s by $\epsilon\in \R^+$. The associated perturbed degree $D^i(\epsilon)=\diag(A^i(\epsilon)\ones)$ and Laplacian matrix $\L^i(\epsilon)=D^i(\epsilon)-A^i(\epsilon)$ are defined accordingly. Indicate the reduced graph $\G^{R_i}$ achieved from $\G$ by removing node $i$ and all its incident edges. Accordingly, $A^{R_i}$ is the adjacency matrix, $D^{R_i}$ is the degree matrix, and $\L^{R_i}$ is the Laplacian matrix of $\G^{R_i}$.


\section{Problem Statement}\label{section:problem_statement}

We study the biconnectivity maintenance problem in multi-robot systems. Communications are assumed to be between each robot and its 1-hop neighbors, or neighbor-to-neighbor data-exchange. The connectivity of the initial network is also presumed.

The following definitions from the algebraic graph theory  will be used in the rest of this paper.
\begin{definition}
	A vertex $i \in \V $ of a connected graph $\G$ is called an \emph{articulation point} if $\G_i^R$ is not connected. 
\end{definition}
\begin{definition}
	A connected graph is called \emph{biconnected} if it has no articulation point.
\end{definition}
\begin{definition}
	A \emph{block} in $\G$ is a maximal induced connected subgraph with no articulation point. If $\G$ itself is connected and has no articulation point, then $\G$ is a block \cite{west2001introduction}.
\end{definition}
\begin{definition}
	If the sub-graph based on node $i$ and its neighbors $\N_i$ is a block, then $i$ is called a locally-biconnected node. 
\end{definition}

 We raise the two following  problems.

\begin{problem}
	For a multi-robot system with a connected interaction graph $\G$, using a distributed algorithm, find if the resulting network graph is biconnected.
\end{problem}

\begin{problem}
	If the network graph is not biconnected, then provide an algorithm to enforce this property.
\end{problem}

The former problem was investigated in the authors' previous work \cite{Zareh2016biconnectivitycheck}. In this paper, we focus on the latter. In other words, we develop a decentralized algorithm to bring a connected network into a biconnected one, i.e., the robots keep their connectivity even if one of them, for any reason, fails to communicate with the others.

\section{Main Contribution}\label{section:main_results}
	 To enable the robots to achieve a biconnected network graph, they must be aware of their connectivity status in the graph, when the corresponding nodes on the network graph and all the incident edges are disconnected. If the graph remains connected in the case of robot $i$ failure, then the node $i$ in the graph is not an articulation point. By putting weakly connected links between node $i$ and its neighbors, we aim at providing an estimate of the condition after a complete disconnection. This was proven in our previous work \cite{Zareh2016biconnectivitycheck}. We obtained that, if the third-smallest eigenvalue of the Laplacian matrix, for a nearly disconnected network, at any locally biconnected node $i$ and for some small $\epsilon\in \R$, meets the following condition
 		\begin{equation}\label{eq:lambda_3_epsilon_bound}
 		\lambda_3(\L^i(\epsilon)) >\epsilon \sqrt{n}\  (\sum\limits_{k=1}^{n}{a}^2_{ik})^{1/2},
 		\end{equation}
 	then the resulting system is biconnected.	
 	If this condition does not hold, then in order to obtain a biconnected graph, the third-smallest eigenvalue must increase. In order to increase this value, we will hereafter define a decentralized protocol based on gradient descent. Note that
 	$$\lambda_3(\L^i(\epsilon))=v_3^T(\L^i(\epsilon)) \L^i(\epsilon) v_3(\L^i(\epsilon)),$$
 	To obtain the gradient, an estimate of $v_3(\L^i(\epsilon))$ is  required.

	 In this way, our approach to solve the biconnectivity problem contains the following steps
 \begin{enumerate}[a)]
 	\item First, using the algorithm introduced in \cite{franceschelli2013decentralized}, we estimate the third-smallest eigenvalue of the Laplacian matrix.
 	\item Then, we propose a decentralized consensus estimator to obtain the   eigenvector associated with the third-smallest eigenvalue of the Laplacian matrix. 
 	
 	\item Finally, using a decentralized gradient-based protocol, the increment of the third-smallest eigenvalue is ensured.
 \end{enumerate}   

The next section provides one of the key results of this paper.

\subsection{Eigenvector estimation}
 In this section, we introduce an estimation protocol to obtain any eigenvector of the Laplacian matrix. These results can be specified to obtain the eigenvector associated with the third-smallest eigenvalue of the Laplacian matrix. 
 
 Assume that, in a multi-robot system, the network graph $\G$ is connected. Let $\tilde{\L}=\L-\tilde{\lambda}(\L) I$, with $\tilde{\lambda}(\L)\in \{\lambda_1(\L),\ldots,\lambda_n(\L)\}$. The eigenspaces of $\L$ and $\tilde{\L}$ are identical. Specifically, the kernel of $\tilde{\L}$ lies in $\span(\tilde{v}(\L))$. Denote by $\tilde{l}_i\in\R^n$ the $i$-th column of $\tilde{\L}$. Let $\displaystyle P_i=\dfrac{\tilde{l}^i\tilde{l}^{i^T}}{\tilde{l}^{i^T}\tilde{l}^i}$, $i=1,\ldots,n$, and define a block-diagonal matrix $P=\diag(P_1,\ldots,P_n)$.

Consider the following distributed estimator
\begin{equation}
    \dot{z}_i(t)=\sum\limits_{j=1}^{n}a_{ij}(z_j(t)-z_i(t))-P_iz_i(t), \ i=1,\ldots, n,					
\end{equation}
in which $z_i, \ i=1,\ldots,n$ is the $i$-th agent's estimation vector for $\tilde{v}(\L)$. We can rewrite the above equation in state-space form as
\begin{equation}\label{eq:z(t)}
	\dot{z}(t)=-Mz(t),
\end{equation}
in which $z=[z_1^T,\ldots,z_n^T]^T$, $M=k(\L\otimes I+P)$, with $k\in \R^+$ being the estimator gain.

The following assumption will be used in the rest of this paper.   
\begin{assumption}\label{assump:simple}
	All eigenvalues of the Laplacian matrix are simple.
\end{assumption}

\begin{remark}
	Note that the elements of $\L$ are functions of the relative distances. Since the robot's positions are supposed to be random, the elements can get any real value. Accordingly, $\L$ is a doubly-stochastic unstructured matrix. Therefore, the eigenvalues and the elements of any eigenvector of $\L$ are almost surely distinct.  If the distances, in some applications, get equal values, then we can define random edge weights to ensure Assumption \ref{assump:simple}. Therefore, the above assumption is not restrictive. This will be verified later by simulations.
\end{remark}

The next lemmas demonstrate some properties of $P$ and $\L\otimes I$.
\begin{lemma}
	Matrix $P$ has all eigenvalues equal to $0$ and $1$.
\end{lemma}
\begin{proof}
	From the definition of $P$, we can show that 
	$$P\cdot P=\diag(P_1\cdot P_1,\ldots,P_n\cdot P_n).$$
	We have 
	$$P_i\cdot P_i=\dfrac{\tilde{l}^i\tilde{l}^{i^T}}{\tilde{l}^{i^T}\tilde{l}^i}\cdot \dfrac{\tilde{l}^i\tilde{l}^{i^T}}{\tilde{l}^{i^T}\tilde{l}^i}=\dfrac{\tilde{l}^i(\tilde{l}^{i^T}\tilde{l}^i)\tilde{l}^{i^T}}{\tilde{l}^{i^T}\tilde{l}^i(\tilde{l}^{i^T}\tilde{l}^i)}=\dfrac{\tilde{l}^i\tilde{l}^{i^T}}{\tilde{l}^{i^T}\tilde{l}^i}=P_i.$$
	Consequently
	$$P\cdot P=\diag(P_1,\ldots,P_n)=P.$$
	Therefore $\lambda^2(P)=\lambda(P),$ which gives $\lambda(P)=0$, or $\lambda(P)=1$. 
\end{proof}
In the next lemma, using the fact that any two eigenvectors associated to two different  eigenvalues of $\L$ are perpendicular, from Assumption \ref{assump:simple}, we select a set of orthonormal eigenvectors  $v_1(\L), \ldots, v_n(\L)$.
\begin{lemma}\label{lemma:eig_L_kron}
	The eigenvalues of $\L\otimes I$ are achieved by $n$-times repeating each eigenvalue of $\L$, i.e.,
	$$\lambda_{kn+1}(\L\otimes I)=\ldots=\lambda_{(k+1)n}(\L\otimes I)=\lambda_k(\L), \ k=0,\ldots,n-1.$$
	The set $\{v_k(\L)\otimes v_l(\L),\ l=1,\ldots,n\}$ forms an orthogonal basis for the eigenspace of $\L\otimes I$ associated with $\lambda_k(\L)$, $k=1, \ldots,n$.  
\end{lemma}
\begin{proof}
	It is trivial to show that the eigenvalues of $\L\otimes I$ are $n$ times repeatedly achieved from those of $\L$.
	
	By multiplying $\L\otimes I$ by $v_k(\L)\otimes v_l(\L)$ we get
	$$\begin{array}{rl}
	(\L\otimes I)(v_k(\L)\otimes v_l(\L))&=\L v_k(\L) \otimes I v_l(\L)\\&=\lambda_k(\L) (v_k(\L)\otimes v_l(\L)), 
	 
	\end{array} $$
	which means that $v_k(\L)\otimes v_l(\L)$ is an eigenvector of $\L\otimes I$ associated with $\lambda_k(\L)$. Note that, for $l\ne m $ we have $v_l^T(\L)v_m(\L)=0$. Then
	$$\begin{array}{l}
	(v_k(\L)\otimes v_l(\L))^T(v_k(\L)\otimes v_m(\L)) \\\hspace{2cm} =v_k^T(\L)v_k(\L)\otimes v_l^T(\L)v_k(\L)=0.  	
	\end{array} $$
This shows the eigenvectors orthogonality and completes the proof.
\end{proof}
\begin{coro}
	The set $\{\ones\otimes v_1(\L),\ldots,\ones\otimes v_n(\L)\}$ forms an orthogonal basis for the kernel of $\L\otimes I$. 
\end{coro}

\begin{lemma}\label{lemma:eigvec_P}
	The intersection of the kernels of $\L\otimes I$ and $P$ is $\span(\ones\otimes \tilde{v}(\L))$.
\end{lemma}
\begin{proof}
	It is trivial to show that $\ones\otimes \tilde{v}$ is in the kernel of $P$. Now, by contradiction we prove that there is no other intersection between the kernels. Let $\ones\otimes \psi$, $\psi\in \R^n\notin\span(\tilde{v}(\L))$, be another intersection for the kernels of $\L\otimes I$ and $P$. Then
	$$P(\ones\otimes \psi)=0.$$
	From the definition of $P$ we get that   
	\begin{equation*}
		\frac{1}{\tilde{l}^{i^T}\tilde{l}^i}\tilde{l}^i\tilde{l}^{i^T} \psi=0, \ i=1,\ldots,n.
	\end{equation*}
	This gives that 
	$$\tilde{l}^{i^T}\psi=0,\ i=1,\ldots,n.$$ 
	Since $\span(\tilde{v}(\L))$ is the kernel of $\tilde{\L}$, the above equation does not give any solution for $\psi$. As a consequence, the intersection of the kernels of $\L$ and $P$ is  $\span(\ones\otimes\tilde{v}(\L))$. 
\end{proof}
\begin{lemma}\label{lemma:M}
	For a connected graph $\G$, if all Laplacian's eigenvalues are simple, the matrix $ M $ in \eqref{eq:z(t)} is positive semidefinite, and has a simple null eigenvalue. 
\end{lemma}
\begin{proof}
	For any non-zero normalized vector $x\in \R^{n^2}, \ x^Tx=1$, the Rayleigh quotient \cite{horn2012matrix} of $M$ is defined as
	\begin{equation}
		R(M,x)=x^TM x.
	\end{equation}
Let $\gamma_1\le \ldots\le\gamma_{n^2}$ be eigenvalues of $M$. Since $M$, $\L\otimes I$, and $P$ are symmetric matrices, and hence Hermitian, from the min-max theorem \cite{horn2012matrix} we get
\begin{equation*}\begin{array}{rl}
\gamma_1=\displaystyle \min\{R(M,x): x\ne 0\}&= k\min\{R(\L\otimes I,x)\\&+R(P,x):x\ne 0\}\\
&\displaystyle=k\min\{R(\L\otimes I,x):x\ne 0\}\\&\displaystyle+k\min\{R(P,x):x\ne 0\}\\
&=k\lambda_{1}+k\lambda_{P_{min}}=0.
\end{array} 
\end{equation*}	
From Lemma \ref{lemma:eig_L_kron}, we know that $\ones \otimes v_j$, $j=1,\ldots,n$ form an orthogonal basis for the kernel of $\L\otimes I$. From Lemma \ref{lemma:eigvec_P}, we know that the intersection of the kernels of $\L\otimes I$ and $P$ is $\span(\ones\otimes \tilde{v}(\L))$. This means that 
$$R(M,x)=0$$
occurs only for $x=\ones\otimes \tilde{v}(\L)$. In other words, $M$ is a positive semi-definite with the only one null-eigenvalue associated with the one-dimensional kernel $\span(\ones\otimes \tilde{v}(\L))$. 
\end{proof}

The following theorem provides conditions to estimate the eigenvectors of the Laplacian matrix.
\begin{theorem}\label{theorem:consensus}
For the system given in \eqref{eq:z(t)} with a connected undirected graph and a Laplacian matrix with all simple eigenvalues, we get
\begin{equation}\label{eq:consensus}
	\lim\limits_{t\to\infty}z_i(t)=\gamma \tilde{v}(\L), \ \ i=1,\ldots,n, \ \gamma\in \R.
\end{equation}
\end{theorem}
\begin{proof}
     Note that $z(t)=\ones\otimes \tilde{v}(\L)$ is an equilibrium point for the system in \eqref{eq:z(t)}. Consider the following Lyapunov functional
     \begin{equation}
     	V(t)=z^T(t)Mz(t).
     \end{equation} 
     It is easy to show that $z(t)=\ones\otimes \tilde{v}(\L)$ is an equilibrium space. By differentiating with respect to time we get
     \begin{equation}
     \dot{V}(t)=-z^T(t)M^2 z(t).
     \end{equation} 
     From Lemma \ref{lemma:M}, we know that $M$ is a positive semi-definite matrix whose kernel is $\span
     (\ones\otimes \tilde{v}(\L))$. As a consequence, $M^2$ is also a positive semi-definite matrix with only one null-eigenvalue, and $\dot{V}$ gets a value equal to zero only on the equilibrium space. This implies that the system converges along the vector $\ones\otimes \tilde{v}(\L)$. Or
     \begin{equation*}
     	\lim\limits_{t\to\infty} z(t)=\gamma(\ones\otimes \tilde{v}(\L)), \ \gamma\in\R
     \end{equation*}
     which proves \eqref{eq:consensus}.
\end{proof}     
     Theorem \ref{theorem:consensus} says that, if the network graph does not change during a certain time interval, then each agent's estimate, $z_i(t)$, converges to a vector parallel to the eigenvector of the Laplacian matrix. This is a key result of this paper. 

  \begin{remark}
  	The presented algorithms for eigenvalue and eigenvector estimation are independent. However, the latter algorithm requires an estimate of the eigenvalues. Hence, the total estimation time includes the time for eigenvalue estimation plus the time for eigenvector estimation. It is assumed that during this time the network graph remains constant.
  \end{remark}
\subsection{Decentralized gradient construction}
In this section we derive the analytical form of a completely decentralized gradient controller to increase the value of any non-null eigenvalue of the Laplacian matrix. 

For an undirected graph $\G$ with the Laplacian matrix $\L$, consider the eigenvalue problem 
$$\L v(\L)=\lambda(\L) v(\L),$$
with $v(\L)\in\R^n$ being a normalized vector. By multiplying both sides by $v^T(\L)$ we obtain
$$v^T(\L)\L v(\L)=\lambda(\L) v^T(\L)v(\L)=\lambda(\L).$$
Derivation with respect to the node $i$'s position gives
\begin{equation}\label{eq:dlambda/dp}
\begin{array}{rl}
	\dfrac{d\lambda(\L)}{dp_i}=& \dfrac{d(v^T(\L)\L v(\L))}{dp_i}=\dfrac{dv^T(\L)}{dp_i}\L v(\L) \vspace{1em}\\ 
&+v^T(\L)\dfrac{d\L}{dp_i} v(\L)+v^T (\L)\L \dfrac{dv(\L)}{dp_i}.
\end{array} 
\end{equation}
Since $\L$ is symmetric, we know that 
$$v^T(\L)\L \dfrac{dv(\L)}{dp_i}=\dfrac{dv^T(\L)}{dp_i}\L v(\L)=\dfrac{1}{2}\lambda \dfrac{d(v^T(\L)v(\L))}{dp_i}=0.$$     
Then from \eqref{eq:dlambda/dp} we get
\begin{equation}
	\dfrac{d\lambda(\L)}{dp_i}=v^T(\L)\dfrac{d\L}{dp_i} v(\L).
\end{equation}
Consider not a group of single integrator robots, $\dot{p}_i=u_i$, where $p_i$ is the position of the $i$-th robot. Then, we introduce the following gradient-based controller 
\begin{equation}\label{eq:eigenvalue_increase}
u_i=\frac{d\lambda(\L)}{dp_i}=v^T(\L)\frac{d\L}{dp_i} v(\L).
\end{equation}
If the conditions in \eqref{eq:lambda_3_epsilon_bound} does not hold, we use the above gradient control to increase the value of $\lambda_3(\L(\epsilon))$.
\begin{remark}
	The elements of the Laplacian matrix that depend on $p_i$ are the ones in the $i$-th row, and due to symmetry, the $i$-th column. Consequently, the elements of $\dfrac{d\L}{dp_i}$ are all zeros except for the $i$-th row and $i$-th column. This implies that $u_i$ can be computed in a decentralized way.
\end{remark}
Now we are ready to render our decentralized biconnectivity enforcing algorithm. The flowchart in Fig. \ref{fig:flowchart} shows how the biconnectivity algorithm works.

\tikzstyle{decision} = [diamond, draw, fill=blue!20, 
text width=6em, text badly centered, node distance=3cm, inner sep=0pt]
\tikzstyle{block} = [rectangle, draw, fill=blue!20, 
text width=7em, text centered, rounded corners, minimum height=4em]
\tikzstyle{line} = [draw, -latex']
\tikzstyle{cloud} = [draw, ellipse,fill=red!20,node distance=2cm,minimum height=3em]
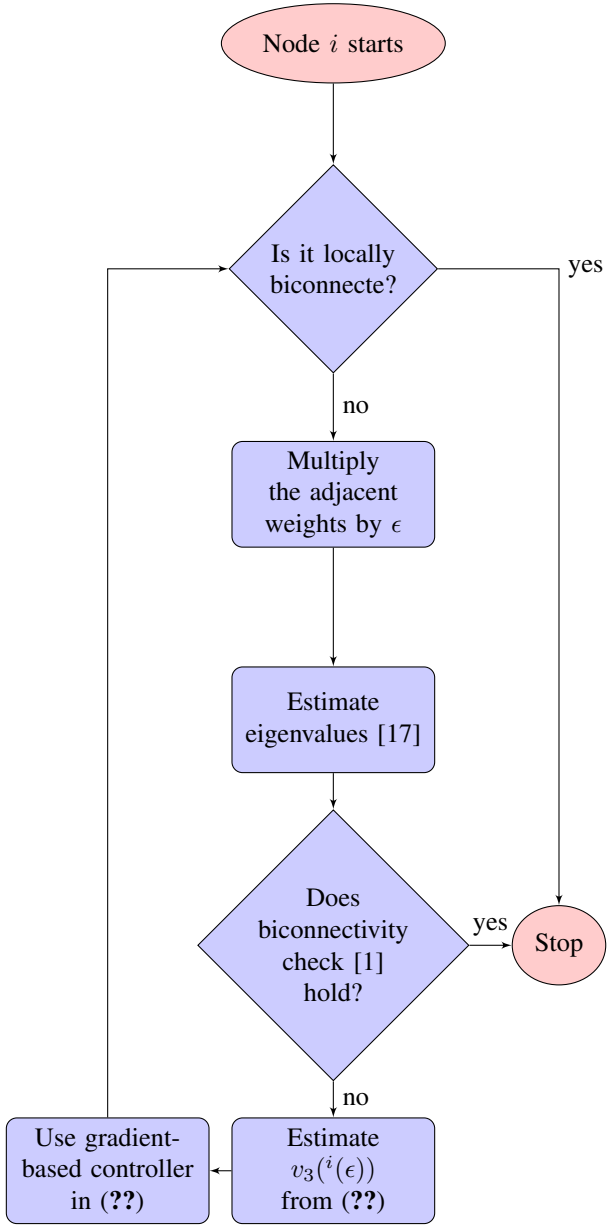
\begin{figure}

\centering
\begin{tikzpicture}[node distance = 3cm, auto]

\node [cloud] (start) {Node $i$ starts};
\node [decision, below of=start] (If_locally) {Is it locally biconnecte?};
\node [block, below of=If_locally] (Perturb) {Multiply the adjacent weights by $\epsilon$};
\node [block, below of=Perturb] (Eig_est) {Estimate eigenvalues \cite{franceschelli2013decentralized}};
\node [decision, below of=Eig_est] (If_bicon) {Does biconnectivity check \cite{Zareh2016biconnectivitycheck} hold?};
\node [block, below of=If_bicon] (Eigvec) {Estimate $v_3(\L^i(\epsilon))$ from \eqref{eq:z(t)}};
\node [block, left of=Eigvec, node distance=3cm] (Gradient) {Use gradient-based controller in \eqref{eq:eigenvalue_increase}};
\node [cloud, right of=If_bicon, node distance=3cm] (Stop) {Stop};

\path [line] (start) -- (If_locally);
\path [line] (If_locally) -- node {no}(Perturb);
\path [line] (If_locally)-| node {yes}(Stop);
\path [line] (Perturb) -- (Eig_est);
\path [line] (Eig_est) -- (If_bicon);
\path [line] (If_bicon) --  node {yes}(Stop);
\path [line] (If_bicon) -- node {no}(Eigvec);
\path [line] (Eigvec) -- (Gradient);
\path [line] (Gradient) |- (If_locally);

\end{tikzpicture}
\caption{Biconnectivity algorithm.} \label{fig:flowchart}
\end{figure}

%
%

\begin{remark}
	Note that the algorithm is separately done by any single robot, and all the included sub-algorithms are based on the local data exchange. Therefore, the whole procedure is totally decentralized.  
\end{remark}
\section{Simulation results}\label{section:simulations}
In this section we aim at showing the effectiveness of the proposed algorithms. We suppose that the communication is defined by the $R$-disk model, in which the elements of the adjacency matrix are defined as 
 \begin{equation}\label{eq:a_ij}
 	a_{ij}=\left\{\begin{array}{lc}
 e^{-(\|p_i-p_j\|^2)/(2\sigma)}& \|p_i-p_j\|\le R  \\ 
 0&  \|p_i-p_j\|> R,
 \end{array} \right.
 \end{equation}
 The selected communication parameters are
 $$R=0.5\ \ \sigma=0.125.$$
 
In the following example, the performance of the eigenvector estimation algorithm is demonstrated.    
 \begin{example}
For the random graph in Fig. \ref{fig:graph1}, the adjacency and Laplacian matrices can be computed from \eqref{eq:a_ij}. 
%

\begin{figure}
	\centering
		\includegraphics[width=1\linewidth,height=0.2\textheight]{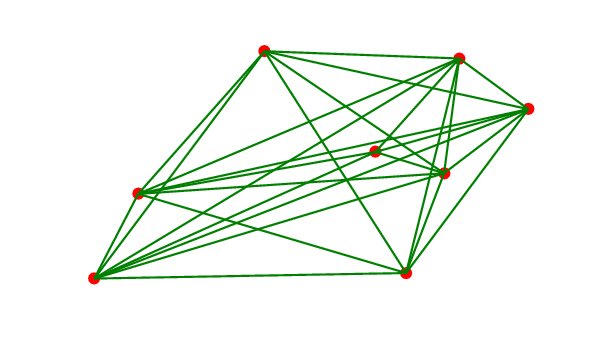}
		\caption{Network graph of a randomly positioned multi-robot system}
		\label{fig:graph1}
\end{figure}
The eigenvalues of the Laplacian matrix are
$$ \{0,\ 0.303,\ 0.412,\ 1.125,\ 1.288,\ 1.721
,\ 2.327,\
2.970\}.$$

Normalized eigenvectors associated with the second and the third-smallest eigenvalues ($\lambda_2=0.303$ and $\lambda_3=0.412$) are
$$\begin{array}{ccc}
v_2=\left[\begin{array}{r}
0,494\\
-0,446\\
-0,311\\
-0,277\\
0,490\\
-0,258\\
0,277\\
0,031 
\end{array}\right]& &v_3=\left[\begin{array}{r}
  0.208\\
  0.178\\
  0.103\\
  0.130\\
  0.069\\
  0.049\\
  0.189\\
  -0.925
\end{array}\right]
\end{array} . $$
The simulation results for $k=50$ and $500$ are shown in Figs.~\ref{fig:v_2} and \ref{fig:v_3}. The corresponding elements of the estimation vectors for different robots are shown with the same colors. We can see that, by use of the proposed algorithm, the state of each robot very rapidly converges to the desired eigenvector. By increasing $k$ the convergence rate increases. We can see that the Assumption~\ref{assump:simple} is true.  

\begin{figure}%
	\centering
	\begin{subfigure}{1\columnwidth}
		\includegraphics[width=\columnwidth,height=0.2\textheight]{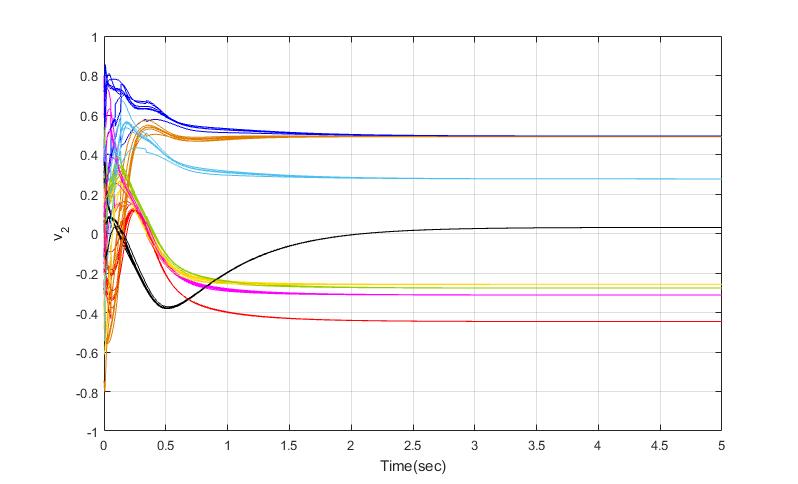}%
		\caption{k=50}%
		\label{subfig:a2}%
	\end{subfigure}\hfill%
	\begin{subfigure}{1\columnwidth}
		\includegraphics[width=1\columnwidth,height=0.2\textheight]{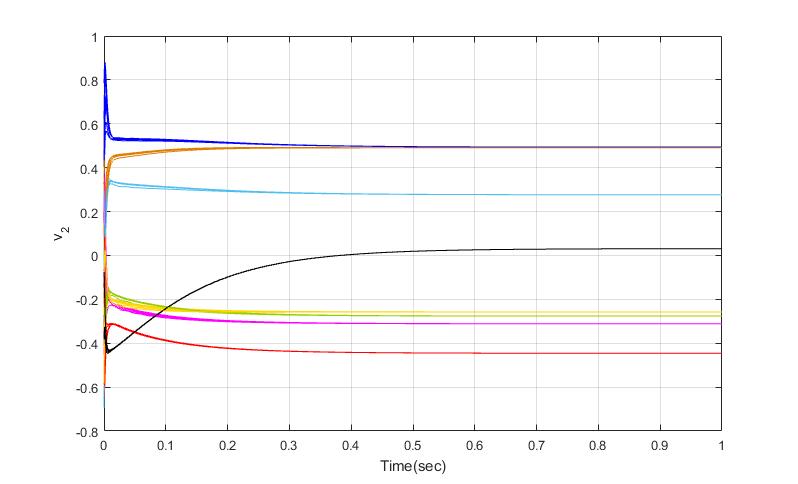}%
		\caption{k=500}%
		\label{subfig:b2}%
	\end{subfigure}
	\caption{Decentralized estimation of $v_2$}
	\label{fig:v_2}
\end{figure}

\begin{figure}%
	\centering
	\begin{subfigure}{1\columnwidth}
		\includegraphics[width=\columnwidth,height=0.2\textheight]{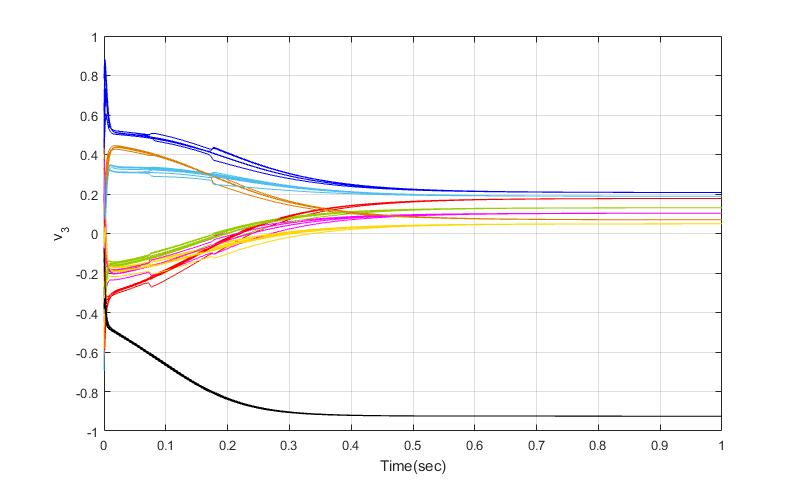}%
		\caption{k=50}%
		\label{subfig:a3}%
	\end{subfigure}\hfill%
	\begin{subfigure}{1\columnwidth}
		\includegraphics[width=1\columnwidth,height=0.2\textheight]{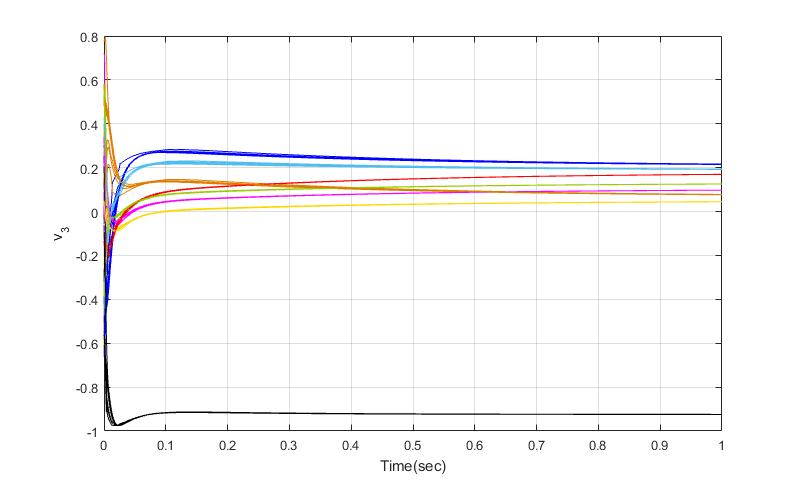}%
		\caption{k=500}%
		\label{subfig:b3}%
	\end{subfigure}
	\caption{Decentralized estimation of $v_3$}
	\label{fig:v_3}
\end{figure}
 \end{example}
The next example demonstrates the results of a consensus problem in a multi-robot system, once with and another time without the biconnectivity algorithm. 
\begin{example}\label{example:2}
	Consider the graph in Fig. \ref{fig:graph2} with $n=8$ nodes. At time zero, the robots start running a simple consensus protocol
	$$\dot{p}_i=u_i^{c},$$
	where $p_i$ indicates the position of the robot $i$, and $u_i^c$ is the local controller
	$$u_i^c=\sum\limits_{j\in\N_i}a_{ij}(p_j-p_i).$$  
	\begin{figure}
		\centering
		\includegraphics[width=1\linewidth,height=0.2\textheight]{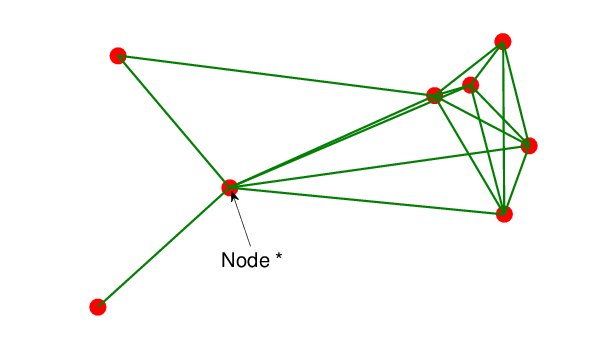}
		\caption{Initial network graph of the multi-robot system in Example \ref{example:2}.}
		\label{fig:graph2}
	\end{figure}
	Fig. \ref{fig:graph_simpleconsensus_withouth_BM} shows that the system gets disconnected after $t=1~sec$. 
	\begin{figure}
\centering
\includegraphics[width=1\linewidth,height=0.2\textheight]{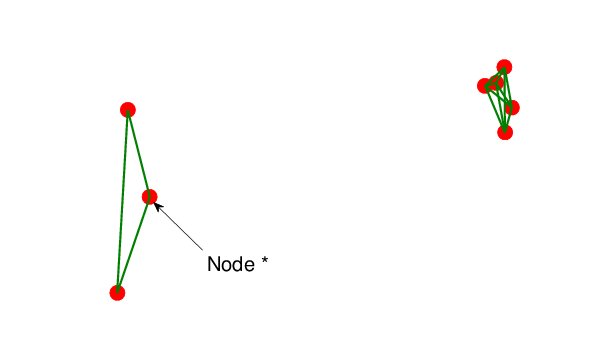}
\caption{Network graph after $t=1~sec$ without biconnectivity algorithm. }
\label{fig:graph_simpleconsensus_withouth_BM}
\end{figure}
Next, the biconnectivity algorithm is utilized. We first check biconnectedness, following the procedure introduced in \cite{Zareh2016biconnectivitycheck}. Note that the node $*$ in Fig. \ref{fig:graph2}, is the only not locally biconnected one, and
hence, it must meet the sufficient conditions introduced in \eqref{eq:lambda_3_epsilon_bound}. Based on the criterion \eqref{eq:lambda_3_epsilon_bound}, and multiplying the weight of the node $*$ by $\epsilon=0.05$, we get
$$\lambda_3(\L^*(\epsilon))=0.013<\epsilon\sqrt{(8)}\sum_{k=1}^{n}a_{*k}=0.0022,$$
which implies that the biconnectivity check fails.
In order to increase the value of $\lambda_3(\L^*(\epsilon))$, we use the gradient-based controller in \eqref{eq:eigenvalue_increase}, by estimating $v_3(\L^*(\epsilon))$ implementing \eqref{eq:z(t)}. From \eqref{eq:a_ij} and \eqref{eq:eigenvalue_increase}, we obtain the following biconnectivity protocol for node $*$
$$\begin{array}{rl}
u_*^b&=v_3^T(\L^*(\epsilon))\dfrac{d\L^*(\epsilon)}{dp^*}v_3(\L^*(\epsilon))  \\ 
=& -\sum_{j=1}^{n}a_{*j}(v_{3_j}(\L^*(\epsilon))-(v_{3_j}(\L^*(\epsilon)))^2\dfrac{p_*-p_j}{\sigma^2},
\end{array}$$
in which $v_{3_j}(\L^*(\epsilon))$ indicates the $j$-th element of the estimation vector $v_3(\L^*(\epsilon))$. As shown in Fig.~\ref{fig:graph_simpleconsensus_with_BM}, the graph reaches  biconnectivity after 1 second.  
\begin{figure}[t!]
	 \begin{subfigure}[t]{0.5\textwidth}
	 	\centering
	 	\includegraphics[width=1\linewidth,height=0.2\textheight]{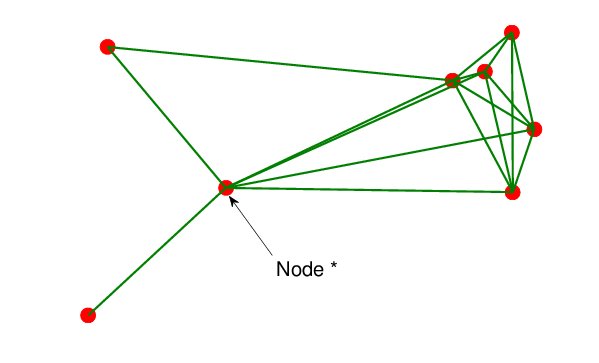}
	 	\caption{$t=0.4$ sec.}
    \end{subfigure}
    \begin{subfigure}[t]{0.5\textwidth}
     	\centering
     	\includegraphics[width=1\linewidth,height=0.2\textheight]{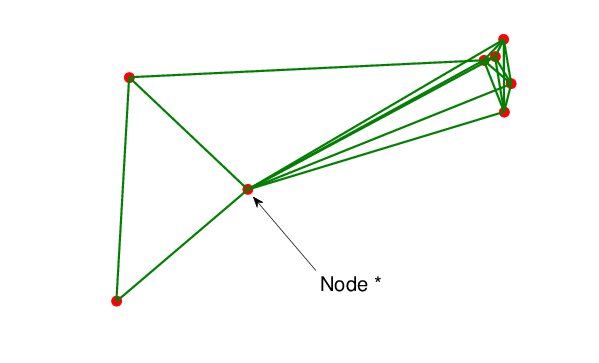}
     	\caption{$t=1$ sec.}
    \end{subfigure}
    \caption{Network graph achieved from biconnectivity algorithm.}
    \label{fig:graph_simpleconsensus_with_BM}
\end{figure}

\end{example}

\section{Conclusions}\label{section:conclusions}
In this paper, we developed a decentralized algorithm to achieve graph biconnectivity in multi-robot systems. We presented an estimation protocol to be executed by every single robot to estimate the eigenvectors of the Laplacian matrix. Simulations showed that, by  increasing the estimation gain, we can expedite the convergence rate.  Using the estimate of the eigenvector, a gradient control was proposed to increase the third-smallest eigenvalue of the Laplacian matrix, to reach the requirements of the biconnectedness, introduced by the authors in \cite{Zareh2016biconnectivitycheck}. In our future work, we aim at finding the convergence rate of the proposed algorithm.
\bibliographystyle{IEEEtran}
 \bibliography{biblio_connectivity,biblio,biblio_applications} 
\end{document}